%% file: main.tex
\documentclass[11pt]{article}
\usepackage{microtype}
\usepackage[margin=1in]{geometry}
\usepackage[switch]{lineno}
\usepackage{tikz}
\usepackage[subpreambles=true]{standalone}
\input{share-preamble}

\usepackage{textcomp}

\usepackage[
  backend=biber,   %
  style=numeric,   %
  sorting=none,    %
  maxbibnames=5,   %
  maxcitenames=2   %
]{biblatex}
\usepackage{graphicx}    %
\usepackage{subcaption}  %
\usepackage{booktabs}    %
\usepackage{tabularx}    %
\usepackage{multirow}    %
\usepackage{multicol}    %
\usepackage[framemethod=tikz]{mdframed}  %

\usepackage{comment}     %

\usepackage{algorithm}
\usepackage{algorithmic}
\newtheorem{theorem}{Theorem}

\usepackage[hidelinks]{hyperref}   %

\title{\textsc{Fairy2i}: Training Complex LLMs from Real LLMs with \\All Parameters in $\{\pm 1, \pm i\}$}
\author{
Feiyu Wang,
Xinyu Tan,
Bokai Huang,
Yihao Zhang,\\
Guoan Wang,
Peizhuang Cong,
Tong Yang\thanks{Corresponding author: \texttt{yangtong@pku.edu.cn}}
\\[0.5em]
\small
Peking University
}

\date{}

\begin{document}
\maketitle

\begin{abstract}
Large language models (LLMs) have revolutionized artificial intelligence, yet their massive memory and computational demands necessitate aggressive quantization, increasingly pushing representations toward the theoretical limit of a single bit.
While complex-valued LLMs, such as iFairy, offer a superior chance for low-bit representation compared to real-valued counterparts, they require training from scratch, preventing the utilization of the vast ecosystem of pre-trained real-valued foundation models.
Here we present \textsc{Fairy2i}, a universal framework that transforms pre-trained real-valued layers into an equivalent widely-linear complex form, enabling extremely low-bit quantization while reusing existing checkpoints.
By proving a lossless mathematical equivalence between real and widely-linear maps, we convert standard Transformers into the complex domain and employ a phase-aware quantization scheme with a highly efficient codebook of fourth roots of unity ($\{\pm 1, \pm i\}$).
Furthermore, we introduce a recursive residual quantization mechanism that iteratively minimizes quantization error, allowing inference to proceed via efficient multiplication-free accumulation.
We demonstrate that \textsc{Fairy2i} restores the performance of LLaMA-2 7B at an effective 2-bit precision to levels nearly comparable with full-precision baselines, significantly outperforming state-of-the-art real-valued binary and ternary quantization methods.
This work bridges the gap between the representational efficiency of complex-valued arithmetic and the practical utility of pre-trained models, paving a new way for efficient inference on commodity hardware. We open-source the \textsc{Fairy2i} model and code at \url{https://huggingface.co/PKU-DS-LAB/Fairy2i-W2} and \url{https://github.com/PKULab1806/Fairy2i-W2}.

\end{abstract}

\section{Introduction}
\label{sec: intro}
Large language models (LLMs) keep growing in width, depth, and sequence length~\cite{achiam2023gpt, touvron2023llama1, touvron2023llama2, dubey2024llama3, yang2025qwen3, guo2025deepseekr1}, pushing memory bandwidth and compute budgets to their limits in both training and deployment~\cite{miao2023inferenceSurvey1, wan2023inferenceSurvey2, kwon2023vllm}.
Quantization seeks to address these bottlenecks by storing weights and activations in a few bits and replacing many floating-point operations with much cheaper integer arithmetic~\cite{zhu2024survey-quant}.
There exist two mainstream workflows for quantization in the literature.
\emph{Post-training quantization (PTQ)}~\cite{frantar2022gptq, lin2023awq, xiao2023smoothquant, chee2023quip, tseng2024quipsharp, liu2024spinquant} is attractive for its simplicity and zero retraining cost. However, at \emph{extremely low bit widths} (e.g., 1–2 bits), it typically suffers from a severe accuracy drop. This failure often stems from compounding factors, including that calibration data are limited, per-tensor or per-group scaling is insufficient to realign codebooks, outlier channels dominate clipping ranges, etc. 
\emph{Quantization-aware training (QAT)}~\cite{liu2023llmqat, chen2024efficientqat, bondarenko2024lowrankqat, liu2025paretoq} inserts the quantizer into the training loop often with straight-through estimators~\cite{bengio2013ste}, allowing the model to adapt to the quantized pattern. While current QAT methods outperform the PTQ methods under extremely low bit budgets, training a model from scratch with QAT is often prohibitively time-consuming and unstable since QAT methods still operate with tiny codebooks, offering insufficient freedom to match layer statistics.

Recently, the \emph{iFairy} framework~\cite{wang2025ifairy} demonstrates the efficacy of transitioning from real-valued to complex-valued transformer architecture within the extremely low-bit quantization regime. By concurrently encoding magnitude and phase, complex weights introduce an additional degree of freedom that enhances modeling flexibility and representational capacity, thereby providing a theoretical basis for the superior performance of quantized complex-valued models. Crucially, \emph{iFairy} adopts the fourth roots of unity on the unit circle ($\{\pm 1, \pm i\}$) as a compact codebook. This approach can fully exploit the 2 bit representation budget and hence achieve better alignment with weight statistics compared to real-valued counterparts, such as BitNet~\cite{wang2023bitnet, wang2025bitnetv2, ma2025bitnet2B4T, ma2024bitnetb1.58}, BitCPM4~\cite{team2025minicpm4}, and ParetoQ~\cite{liu2025paretoq}, which typically rely on a ternary set $\{+1, 0, -1\}$, resulting in the underutilization of the available 2-bit encoding space. Consequently, complex-valued Transformer models exhibit significantly improved accuracy retention at extremely low bit precision compared to state-of-the-art real-valued QAT approaches.
Despite these benefits, iFairy models are trained from scratch with the QAT method.
Reproducing such end-to-end training on larger models demands substantial time and compute, which hinders adoption when strong real-valued checkpoints already exist.
In practice, many users require a path that \emph{reuses} existing real-valued checkpoints and still obtains complex-domain benefits at an extremely low bit width.

To address this problem, we propose \method{}, a universal framework for extremely-low bit QAT which converts a real-valued model into a complex-valued computation pattern and quantizes the converted model with efficient quantization paradigms in the complex domain. 
Specifically, we keep the pretrained real model, re-parameterize each real-valued linear layer in an \emph{equivalent widely-linear complex form}, then \emph{continue training at extremely low bits in the QAT method} with a simple phase-aware quantizer (e.g., PhaseQuant~\cite{wang2025ifairy}), and further \emph{quantize the residual error} to reduce the remaining gap at a tiny extra bit cost.
The whole framework consists of three components, including \textbf{widely-linear representation}, \textbf{phase-aware complex quantization}, and \textbf{recursive residual error quantization}.
In a widely-linear representation~\cite{widely-linear}, we state that any real linear map on an even-dimensional space can be written exactly as the sum of a complex-linear part and a conjugate-linear part.
We apply this to every linear layer of a real-valued model, generally including $\mathbf{Q}$ projection, $\mathbf{K}$ projection, $\mathbf{V}$ projection, and $\mathbf{O}$ projection in the self-attention block, as well as $\text{Up}$ projection, $\text{Gate}$ projection, and $\text{Down}$ projection in the feed-forward network, taking LLaMA-alike architecture as an example. 
For attention, using the real part of the Hermitian inner product of queries and keys preserves the original scores.
It can be proven that the forward computation is unchanged after the above transformation. 
On top of the widely-linear form, phase-aware complex quantization uses a very small complex codebook \(\{\pm1, \pm i\}\) with shared scaling factors across $\pm 1$  and $\pm i$ per group, respectively.
We then train the transformed model with the QAT method while keeping full-precision master weights, so we can reuse existing real-valued checkpoints.
To further reduce quantization error, we propose recursive residual error quantization, which also quantizes the residual error using the same phase-aware mechanism, recursively.
The deployed weight is the sum of the extremely low bit representation of the original weight and its several successive quantization errors.
This residual quantization trades a tiny extra code (one more 1–2 bits per weight, plus small per-group metadata) for a noticeable reduction in error.

We apply \method{} in off-the-shelf mainstream open source LLMs, such as the LLaMA series. The evaluation results validate the efficiency of the \method{} framework. Our contributions can be summarized as follows:
\begin{itemize}
  \item We propose a universal method to transform any real layer into an equivalent widely-linear complex form, keeping the pre-quantization model behavior unchanged.
  \item We propose to continue training at extremely low bits with a phase-aware quantizer, reusing real-valued checkpoints.
  \item We propose to further reduce quantization error by quantizing the residual with the same mechanism, yielding a compact sum of a few low-bit terms, which can significantly reduce the quantization error.
\end{itemize}

The rest of the paper is organized as follows.
Section~\ref{sec: related} reviews related work on quantization and complex-valued neural networks.
Section~\ref{sec: methods} formalizes the complex-valued reparameterization, phase-aware quantizer, and the recursive residual error quantization and introduces the \method{} framework comprehensively.
Section~\ref{sec: eval} presents empirical results on standard backbones.
Section~\ref{sec: conclusion} concludes this paper and discusses future work.

\section{Related Work}
\label{sec: related}

\subsection{Quantization for Large Language Models}
Model quantization has become a de facto standard for compressing Large Language Models (LLMs). Existing methods can be broadly categorized into Post-Training Quantization (PTQ)~\cite{frantar2022gptq} and Quantization-Aware Training (QAT).

\paragraph{Post-training quantization (PTQ).} 
PTQ methods aim to compress models without extensive retraining, often using a small calibration dataset. Early works like AWQ~\cite{lin2023awq} and GPTQ~\cite{frantar2022gptq} focused on preserving the weights' output activation distribution, achieving impressive results at 3-4 bits. More recent approaches like OmniQuant~\cite{shao2023omniquant}, QuIP~\cite{chee2023quip}, QuIP\#~\cite{tseng2024quipsharp} attempt to push the limit to 2 bits. However, PTQ methods typically suffer from severe performance degradation at extremely low bit-widths (e.g., $\le 2$ bits) because the limited discrete codebook of real numbers (e.g., binary quantization of $\{+1, -1\}$ or ternary quantization of $\{+1, 0, -1\}$) cannot adequately approximate the heavy-tailed distribution of LLM weights without adapting the network parameters themselves.

\paragraph{Extremely low-bit QAT and 1-bit architectures.}
To address the limitations of PTQ, QAT methods fine-tune the model parameters to adapt to the quantization noise.
BitNet and its successor BitNet b1.58 introduced a 1-bit (binary $\{+1, -1\}$) and 1.58-bit (ternary $\{+1, 0, -1\}$) architecture, respectively, demonstrating that LLMs can retain performance at extremely low precision if trained from scratch~\cite{liu2025binaryNNSurvey}.
While effective, these methods often require training massive models from random initialization, incurring prohibitive computational costs.
Although some works attempt to distill full-precision models into low-bit counterparts (e.g., LLM-QAT), the rigid topology of real-valued binary/ternary weights fundamentally limits the model's capacity and optimization landscape, leading to instability or convergence issues when strictly enforced on pre-trained weights.

\subsection{Complex-Valued Neural Networks}
The integration of complex numbers into neural networks boasts a rich history, predominantly within domains where data possess intrinsic phase and magnitude, such as signal processing and imaging~\cite{hirose2006complex, trabelsi2017zReLU, lee2022complex-survey, bassey2021survey-survey2, yang2020complextransformer, eilers2023complextransformer2}. By representing weights and activations in the complex domain, Complex-Valued Neural Networks (CVNNs) offer the theoretical potential to capture more intricate patterns and dependencies than their real-valued counterparts.
However, despite this promise, the application of CVNNs to natural language processing—and to Large Language Models (LLMs) in particular—remains remarkably scarce.

\paragraph{The iFairy approach.}
Recently, the \textit{iFairy}~\cite{wang2025ifairy} framework demonstrated the potential of complex-valued quantization for LLMs. By utilizing a phase-aware codebook $\{\pm 1, \pm i\}$ (the fourth roots of unity), \textit{iFairy} showed that complex-valued Transformers could outperform real-valued binary networks at equivalent bit budgets.
However, a major bottleneck of \textit{iFairy} and similar methods is the lack of compatibility with existing real-valued pre-trained models like LLaMA~\cite{touvron2023llama1, touvron2023llama2, dubey2024llama3}, OPT~\cite{zhang2022opt}, and Qwen~\cite{yang2025qwen3}. They typically require training a complex-valued model from scratch, which is impractical for most practitioners given the scale of modern LLMs.

\paragraph{Our contribution.}
\method{} bridges this gap by introducing a \textit{widely-linear representation} that mathematically equates a real-valued layer to a complex-valued one. This allows us to leverage the superior quantization topology of the complex domain, specifically the $\{\pm 1, \pm i\}$ codebook, while initializing directly from powerful real-valued checkpoints, thereby combining the efficiency of low-bit quantization with the accessibility of pre-trained LLMs.

\section{Methods}
\label{sec: methods}
In this section, we introduce the \method{} framework comprehensively.
Our goal is to start from a pretrained \emph{real-valued} model, keep its forward behavior intact before quantization, and then continue training at \emph{extremely low bit widths} with a simple and stable complex quantizer. The method proceeds in three steps that build upon each other.

\textbf{Step 1: Widely-linear transformation (Section~\ref{sec:lin-widely}).}
We first express each real linear layer in an equivalent \emph{widely-linear complex form}.
This transformation is exact and unique, and one can use the real and imaginary parts of complex-valued weights to recover the original real layer, and conversely. Before quantization, the network computes the same outputs as the original real model.

\textbf{Step 2: Phase-aware complex quantization (Section~\ref{sec:quant-ifairy}).}
We then quantize the complex parameters using a phase-based scheme on the unit circle with the fixed codebook $\{\pm 1, \pm i\}$. Each weight is projected to the nearest codeword by angle, and two axis-wise scaling factors (real and imaginary, respectively) restore the magnitude.
During QAT, we maintain full-precision master weights, use the quantized copies in the forward pass, and propagate gradients via a straight-through estimator (STE).
This step is simple to implement and applies identically to all widely-linear layers.

\textbf{Step 3: Recursive residual error quantization (Section~\ref{sec:residual-quant}).}
To further reduce error with a minimal extra budget, we quantize the residual error after the first pass.
At each stage, we project the current residual error to the same small codebook $\{\pm 1, \pm i\}$, subtract the quantized correction, calculate the new residual error, and repeat for a small number of stages.
The finally deployed weight is the sum of these few low-bit terms.

In short, Section~\ref{sec:lin-widely} provides the exact equivalence that allows us to reuse real checkpoints without altering pre-quantization behavior; Section~\ref{sec:quant-ifairy} introduces a stable, extremely low-bit quantizer in the complex plane; and Section~\ref{sec:residual-quant} refines the approximation by recursively quantizing the residual error using the same mechanism.

\subsection{Widely-linear transformation from real domain to complex domain}
\label{sec:lin-widely}
\paragraph{Setup and notation.}
We use column vectors. Let the real input and output be $\tilde{\mathbf{x}}\in\mathbb{R}^{2m}$ and $\tilde{\mathbf{y}}\in\mathbb{R}^{2n}$.
The original linear layer of real-valued models can be represented as:
\begin{equation}
\label{eq:real-layer}
\tilde{\mathbf{y}}=\mathbf{R}\,\tilde{\mathbf{x}},
\qquad
\mathbf{R}\in\mathbb{R}^{(2n)\times(2m)}.
\end{equation}
We pair the last dimension as real/imaginary components and write:
\[
\tilde{\mathbf{x}}=\begin{bmatrix}\RePart\mathbf{x}\\ \ImPart\mathbf{x}\end{bmatrix},\quad
\tilde{\mathbf{y}}=\begin{bmatrix}\RePart\mathbf{y}\\ \ImPart\mathbf{y}\end{bmatrix},
\quad
\mathbf{x}=\RePart\mathbf{x}+i\,\ImPart\mathbf{x}\in\mathbb{C}^{m},\quad
\mathbf{y}=\RePart\mathbf{y}+i\,\ImPart\mathbf{y}\in\mathbb{C}^{n}.
\]
Our goal is to express the same linear map in a widely-linear form:
\begin{equation}
\label{eq:widely-linear}
\mathbf{y}=\mathbf{U}\,\mathbf{x}+\mathbf{W}\,\overline{\mathbf{x}}
\qquad (\mathbf{U},\mathbf{W}\in\mathbb{C}^{n\times m}),
\end{equation}
and to provide an explicit, \emph{lossless} correspondence between $\mathbf{R}$ and $(\mathbf{U},\mathbf{W})$.

\paragraph{Lossless real to complex reparameterization.}
\label{sec:reparam}
Partition $\mathbf{R}$ into $n\times m$ real blocks:
\[
\mathbf{R}=\begin{bmatrix}
\mathbf{R}_{11} & \mathbf{R}_{12}\\[2pt]
\mathbf{R}_{21} & \mathbf{R}_{22}
\end{bmatrix},\qquad
\mathbf{R}_{ij}\in\mathbb{R}^{n\times m}.
\]
Define:
\begin{equation}
\label{eq:UW-from-R}
\begin{aligned}
\RePart\mathbf{U}&=\tfrac{1}{2}(\mathbf{R}_{11}+\mathbf{R}_{22}),&
\ImPart\mathbf{U}&=\tfrac{1}{2}(\mathbf{R}_{21}-\mathbf{R}_{12}),\\[2pt]
\RePart\mathbf{W}&=\tfrac{1}{2}(\mathbf{R}_{11}-\mathbf{R}_{22}),&
\ImPart\mathbf{W}&=\tfrac{1}{2}(\mathbf{R}_{12}+\mathbf{R}_{21}),
\end{aligned}
\end{equation}
and set $\mathbf{U}=\RePart\mathbf{U}+i\,\ImPart\mathbf{U}$, $\mathbf{W}=\RePart\mathbf{W}+i\,\ImPart\mathbf{W}$.

\begin{theorem}[Exact equivalence]
For every $\mathbf{x}\in \mathbb{C}^m$,
\begin{equation}
\label{eq:equivalence}
\tilde{\mathbf{y}}=\mathbf{R}\,\tilde{\mathbf{x}} \quad\Longleftrightarrow\quad
\begin{bmatrix}\RePart \mathbf{y}\\ \ImPart \mathbf{y}\end{bmatrix}
=\mathbf{R}\begin{bmatrix}\RePart \mathbf{x}\\ \ImPart \mathbf{x}\end{bmatrix}
\quad\Longleftrightarrow\quad
\mathbf{y}=\mathbf{U}\mathbf{x}+\mathbf{W}\overline{\mathbf{x}}.
\end{equation}
Moreover, the pair $(\mathbf{U}, \mathbf{W})$ is unique.
\end{theorem}

\begin{proof}
We utilize two block operators $\mathcal{R}_{\mathrm{lin}}$ and $\mathcal{R}_{\mathrm{conj}}$ that embed complex matrices into real $2\times 2$ blocks:
\begin{align}
\label{eq:realify-ops}
\mathcal{R}_{\mathrm{lin}}(\mathbf{U})
&=
\begin{bmatrix}
\RePart \mathbf{U} & -\,\ImPart \mathbf{U}\\[2pt]
\ImPart \mathbf{U} & \ \ \RePart \mathbf{U}
\end{bmatrix},&
\mathcal{R}_{\mathrm{conj}}(\mathbf{W})
&=
\begin{bmatrix}
\RePart \mathbf{W} & \ \ \ImPart \mathbf{W}\\[2pt]
\ImPart \mathbf{W} & -\,\RePart \mathbf{W}
\end{bmatrix}.
\end{align}
They satisfy, for all $\mathbf{x}$:
\begin{equation}
\label{eq:realify-eval}
\begin{bmatrix}\RePart(\mathbf{U}\mathbf{x})\\ \ImPart(\mathbf{U}\mathbf{x})\end{bmatrix}
=\mathcal{R}_{\mathrm{lin}}(\mathbf{U})
\begin{bmatrix}\RePart\mathbf{x}\\ \ImPart\mathbf{x}\end{bmatrix},
\qquad
\begin{bmatrix}\RePart(\mathbf{W}\overline{\mathbf{x}})\\ \ImPart(\mathbf{W}\overline{\mathbf{x}})\end{bmatrix}
=\mathcal{R}_{\mathrm{conj}}(\mathbf{W})
\begin{bmatrix}\RePart\mathbf{x}\\ \ImPart\mathbf{x}\end{bmatrix}.
\end{equation}
Expanding the right-hand sides of Eq.~\eqref{eq:realify-eval} shows that $\mathcal{R}_{\mathrm{lin}}(\mathbf{U})$ and $\mathcal{R}_{\mathrm{conj}}(\mathbf{W})$ act on stacked real/imaginary parts exactly as the complex maps $\mathbf{U}(\cdot)$ and $\mathbf{W}\overline{(\cdot)}$.
Solving $\mathbf{R}=\mathcal{R}_{\mathrm{lin}}(\mathbf{U})+\mathcal{R}_{\mathrm{conj}}(\mathbf{W})$ blockwise yields Eq.~\eqref{eq:UW-from-R}. Substituting back $\mathbf{U},\mathbf{W}$ defined in Eq.~\eqref{eq:UW-from-R}, we have:
\begin{equation}
\label{eq:R-sum}
\mathbf{R}\;=\;\mathcal{R}_{\mathrm{lin}}(\mathbf{U})+\mathcal{R}_{\mathrm{conj}}(\mathbf{W}),
\end{equation}
and consequently, for every $\mathbf{x}$:
\begin{equation}
\tilde{\mathbf{y}}=\mathbf{R}\,\tilde{\mathbf{x}} \quad\Longleftrightarrow\quad
\begin{bmatrix}\RePart \mathbf{y}\\ \ImPart \mathbf{y}\end{bmatrix}
=\mathbf{R}\begin{bmatrix}\RePart \mathbf{x}\\ \ImPart \mathbf{x}\end{bmatrix}
\quad\Longleftrightarrow\quad
\mathbf{y}=\mathbf{U}\mathbf{x}+\mathbf{W}\overline{\mathbf{x}}.
\end{equation}
Uniqueness follows because Eq.~\eqref{eq:UW-from-R} is the only blockwise solution.
\end{proof}

\paragraph{Self-attention in widely-linear form.}
\label{sec:attn}
For Transformer projections, we apply Eq.~\eqref{eq:UW-from-R} independently to $\mathbf{R}_Q,\mathbf{R}_K,\mathbf{R}_V,\mathbf{R}_O$.
Let $\mathbf{q},\mathbf{k},\mathbf{v}$ be the complex-valued outputs of the Q/K/V projections.
We use Hermitian scores:
\begin{equation}
\label{eq:hermitian-score}
\mathbf{S}=\frac{1}{\sqrt{d_k}} \RePart\!\left(\mathbf{q}^{\mathrm{H}}\mathbf{k}\right),
\end{equation}
where $(\cdot)^{\mathrm{H}}$ denotes the conjugate transpose.
This formulation equals the original real dot-product scores when applied to stacked representations, i.e., $\RePart(\mathbf{q}^{\mathrm{H}}\mathbf{k}) = \RePart(\mathbf{q})^\top \RePart(\mathbf{k}) + \ImPart(\mathbf{q})^\top \ImPart(\mathbf{k}) = \tilde{\mathbf{q}}^{\top}\tilde{\mathbf{k}}$.
The softmax and value aggregation act separately on the real and imaginary parts, as the attention score is a real number; the output projection uses the same widely-linear mapping. Note that the self-attention in a widely-linear form is compatible with highly-optimized FlashAttention kernels~\cite{dao2022flashattention, dao2023flashattention2, shah2024flashattention3} since we can unstack the complex-valued $\mathbf{q},\mathbf{k},\mathbf{v}$ into a longer real-valued vector, respectively.

\paragraph{Remarks.} Note that we assume the real linear matrix $\mathbf{R}\in\mathbb{R}^{(2n)\times(2m)}$ has even dimensions.
If the real dimension is odd, we can pad one zero channel to make it even, apply the mapping, and discard the padded output channel. Figure~\ref{fig:transformation} illustrates the widely-linear transformation from a real-valued linear layer to a complex-valued linear layer.

\begin{figure}[!th]
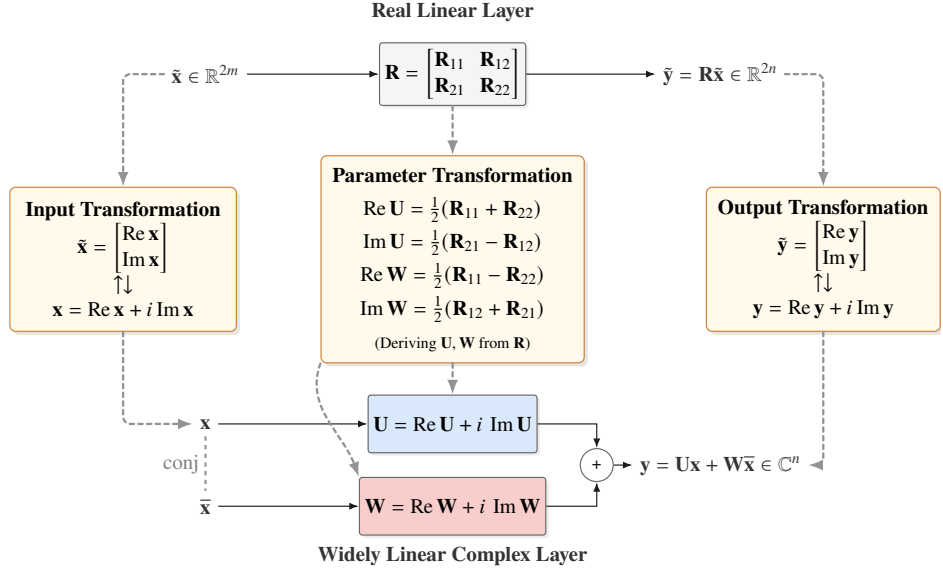

    \centering    \includestandalone[width=0.75\textwidth]{Figure/Fairy2i-transform}
    \caption{Illustration of a widely-linear transformation from the real domain to the complex domain.}
    \label{fig:transformation}
\end{figure}

\subsection{Complex-valued low-bit quantization}
\label{sec:quant-ifairy}

We quantize the complex weights that appear in the widely-linear form (both $\mathbf{U}$ and $\mathbf{W}$) using a simple, phase-based rule called PhaseQuant on the unit circle, adapted from \emph{iFairy}. The codebook is:
\[
\mathcal{S}_{2\text{-bit}}=\{\pm 1, \pm i\}.
\]

\paragraph{Deterministic projection by phase.}
For a weight $w\in\mathbb{C}$, we choose the nearest codeword by angle:
\begin{equation}
\label{eq:ifairy-proj}
b(w)\;=\;\arg\max_{s\in \mathcal{S}_{2\text{-bit}}}\ \RePart\!\big(w\,\overline{s}\big)
\quad\in\ \{\pm1,\pm i\}.
\end{equation}
Equivalently, with $\theta=\operatorname{Arg}(w)\in(-\pi,\pi]$ and $k=\left\lfloor \frac{2\theta}{\pi}+\tfrac12 \right\rfloor$,
\(
b(w)=i^{\,k}.
\)
We write $b(w)=b_{\mathrm{re}}+i\,b_{\mathrm{im}}$ with $b_{\mathrm{re}},b_{\mathrm{im}}\in\{-1,0,+1\}$.

\paragraph{Axis-wise scaling factors.}
To recover magnitude, we estimate two scales from the projected axes:
\begin{equation}
\begin{aligned} 
\label{eq:ifairy-scale}
s_{\mathrm{re}}
\;=\;
\frac{1}{\big|\{w:\,b(w)\in\{\pm1\}\}\big|}
\sum_{b(w)\in\{\pm1\}}\big|\,\RePart(w)\big|,\\
s_{\mathrm{im}}
\;=\;
\frac{1}{\big|\{w:\,b(w)\in\{\pm i\}\}\big|}
\sum_{b(w)\in\{\pm i\}}\big|\,\ImPart(w)\,\big|.
\end{aligned}
\end{equation}
These are per-tensor averages\footnote{In this paper, we only use per-tensor scaling factors and we leave the investigation of various group strategies as future work}.

\paragraph{Dequantization.}
Each weight is replaced in the forward pass by:
\begin{equation}
\label{eq:ifairy-dequant}
\hat w \;=\; s_{\mathrm{re}}\, b_{\mathrm{re}} \;+\; i\, s_{\mathrm{im}}\, b_{\mathrm{im}}.
\end{equation}

\paragraph{QAT training.}
We maintain full-precision masters and use Eq.~\eqref{eq:ifairy-proj}--\eqref{eq:ifairy-dequant} only in the forward pass. Gradients flow through the masters via a straight-through estimator (STE), i.e., we treat $b(\cdot)$ as the identity function in the backward pass within a small neighborhood. The scales in Eq.~\eqref{eq:ifairy-scale} are recomputed at each step and are differentiable w.r.t.\ the masters via their dependence on $w$.

\paragraph{Remarks.}
(i) The codebook $\{\pm1,\pm i\}$ utilizes the full 2-bit budget and is symmetric on the unit circle, aligning well with complex geometry and attention scoring.
(ii) The procedure applies elementwise to $\mathbf{U}$ and $\mathbf{W}$; it is independent of the widely-linear equivalence and can be applied to any complex layer. (iii) To alleviate the computational overhead during training, we adopt \emph{Gauss's multiplication algorithm}. Standard complex multiplication $(a+ib)(c+id)$ requires 4 real multiplications. Gauss optimization reformulates this as:
\[
\RePart = ac - bd, \qquad \ImPart = (a+b)(c+d) - ac - bd.
\]
This computation requires only 3 real multiplications and 5 additions. In the context of large-scale Transformer layers where matrix multiplication dominates, this reduction from 4 to 3 real multiplications per element provides a theoretical $25\%$ reduction in FLOPs for the GEMM operations in both forward and backward passes.

\subsection{Recursive residual quantization}
\label{sec:residual-quant}

We further reduce quantization error by repeatedly quantizing the \emph{residual} left after the initial pass.
The core concept is straightforward: represent each complex weight as a short sum of very low-bit terms, where each new term corrects the approximation error of the previous ones.
We apply this to the complex weights in the widely-linear form (both $\mathbf{U}$ and $\mathbf{W}$), using the same iFairy-style projection and axis-wise scaling as in Section~\ref{sec:quant-ifairy}.

Intuitively, consider each complex weight as a point in the complex plane.
The first step projects it to one of four axis directions ($\pm1,\pm i$) and applies a scale for that axis, providing a coarse approximation.
The residual vector points from this approximation to the true weight.
We then project the residual to the same small codebook and scale it.
Adding this correction brings the approximation closer to the target.
Each subsequent stage repeats this process on the remaining residual; accuracy improves, and the residual magnitude decreases.

\paragraph{Formal definition of recursive residual quantization.}
Let $W\in\mathbb{C}^{n\times m}$ be a matrix of complex weights.
Fix a small codebook $\mathcal{S}$, e.g., $\mathcal{S}=\{\pm1,\pm i\}$ for 2-bit quantization.
Choose a small number of stages $T\in\{1,2,\ldots\}$.
Initialize $R^{(0)}\coloneqq W$ and, for $t=0,1,\dots,T-1$, perform:

\begin{itemize}
    \item \textbf{Quantization:}
    Apply PhaseQuant to $R^{(t)}$ to obtain $b(w)$ for each $w \in R^{(t)}$ according to Eq.~\eqref{eq:ifairy-proj}. Let $B^{(t)}=B^{(t)}_{\mathrm{re}}+i\,B^{(t)}_{\mathrm{im}}$. During this process, we also compute the scaling factors $s^{(t)}_{\mathrm{re}}$ and $s^{(t)}_{\mathrm{im}}$.
    \item \textbf{Dequantization and residual update:}
    \[
    \widehat{W}^{(t)} \;=\; s^{(t)}_{\mathrm{re}}\,B^{(t)}_{\mathrm{re}} \;+\; i\,s^{(t)}_{\mathrm{im}}\,B^{(t)}_{\mathrm{im}},
    \qquad
    R^{(t+1)} \;=\; R^{(t)} - \widehat{W}^{(t)}.
    \]
\end{itemize}
After $T$ stages, the deployed approximation is the finite sum:
\begin{equation}
\label{eq:res-sum}
W_q \;\approx\; \sum_{t=0}^{T-1} \widehat{W}^{(t)}.
\end{equation}
When $T=1$, this reduces to the basic PhaseQuant quantization. For $T=2$, one extra low-bit term is added per weight, typically removing most of the remaining error. Figure~\ref{fig:recursive} shows the workflow of recursive residual quantization with $T=2$.

\begin{figure}
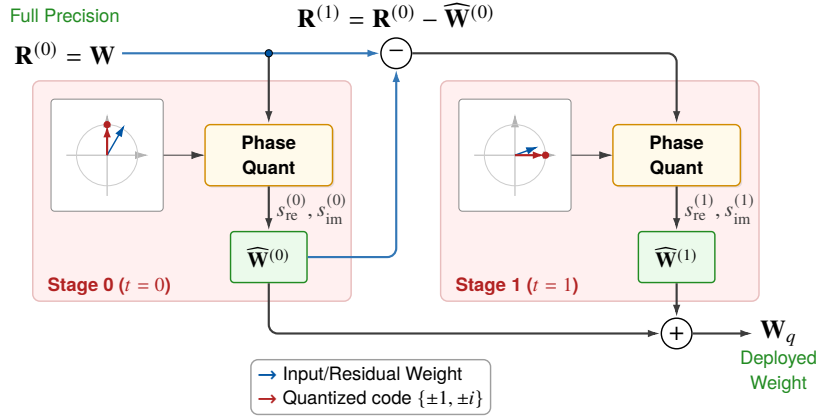

    \centering
    \includestandalone[width=0.7\textwidth]{Figure/recursive-darker}
    \caption{Illustration of recursive residual quantization with $T=2$.}
    \label{fig:recursive}
\end{figure}

\paragraph{Training with QAT.}
We maintain full-precision masters for $W$ and compute $\sum_{t}\widehat{W}^{(t)}$ in the forward pass.
The projection $b(\cdot)$ uses the hard nearest-codeword rule; in the backward pass, we apply a straight-through estimator so that gradients flow to the masters.
Scales are recomputed per step and are differentiable through their dependence on the residuals.
We apply the same procedure independently to $\mathbf{U}$ and $\mathbf{W}$ in every widely-linear layer.

\subsection{Storage Efficiency and Parallel Multiplication-Free Inference}
\label{sec:inference-analysis}

\paragraph{Storage efficiency.}
\label{sec:inference-analysis:storage}
\method{} achieves extreme storage compactness by utilizing the codebook $\mathcal{S}=\{\pm 1, \pm i\}$. Since $|\mathcal{S}|=4$, each quantized complex weight can be exactly encoded using 2 bits.
Consider the original real-valued weight matrix $\mathbf{R} \in \mathbb{R}^{(2n)\times(2m)}$, which contains a total of $4nm$ real parameters. Through our widely-linear transformation, $\mathbf{R}$ is reparameterized into two complex matrices $\mathbf{U}, \mathbf{W} \in \mathbb{C}^{n\times m}$, resulting in a total of $2nm$ complex weights.
The total storage cost for the quantized model (at stage $T=1$) is therefore $2nm \times 2 \text{ bits} = 4nm \text{ bits}$.
Dividing this by the original parameter count ($4nm$), \method{} effectively consumes exactly \textbf{1 bit per real parameter} on average.
When extending to recursive residual quantization with $T$ stages, the total memory footprint scales linearly as \textbf{$T$ bits per real parameter}.
For example, our primary configuration with $T=2$ occupies 2 bits per real parameter.

\paragraph{Parallel stagewise inference algorithm.}
The inference process for a widely-linear layer with $T$ recursive stages is detailed in Algorithm~\ref{alg:inference}. 
Since the discrete weight components $B^{(t)}_{\mathrm{re}}$ and $B^{(t)}_{\mathrm{im}}$ only contain values from $\{-1, 0, +1\}$, the matrix-vector products reduce to conditional accumulators. Specifically, multiplying by $+1$ is an addition, $-1$ is a subtraction, and $0$ is a skip. Multiplying by $\pm i$ simply involves swapping the real and imaginary components of the input activation $\mathbf{x}$ with a sign flip.
Crucially, the loop over stages $t=0 \dots T-1$ (lines 3--6) is data-independent and can be executed in parallel streams.

\begin{algorithm}[htbp]
\caption{Parallel Inference for \method{} Layer}
\label{alg:inference}
\begin{algorithmic}[1]
\REQUIRE Input $\mathbf{x} = \mathbf{x}_{\mathrm{re}} + i\mathbf{x}_{\mathrm{im}} \in \mathbb{C}^m$
\REQUIRE Quantized Codes $\{B^{(t)}_{\mathrm{re}}, B^{(t)}_{\mathrm{im}}\}_{t=0}^{T-1}$ where entries $\in \{-1, 0, 1\}$
\REQUIRE Scales $\{s^{(t)}_{\mathrm{re}}, s^{(t)}_{\mathrm{im}}\}_{t=0}^{T-1}$
\ENSURE Output $\mathbf{y} \in \mathbb{C}^n$
\STATE Initialize output accumulator $\mathbf{y} \leftarrow \mathbf{0}$
\STATE \textbf{parallel for} $t = 0$ to $T-1$ \textbf{do}
    \STATE \quad \textit{// Multiplication-free operations (Add/Sub only)}
    \STATE \quad $\mathbf{v}_{\mathrm{re}} \leftarrow B^{(t)}_{\mathrm{re}}\mathbf{x}_{\mathrm{re}} - B^{(t)}_{\mathrm{im}}\mathbf{x}_{\mathrm{im}}$
    \STATE \quad $\mathbf{v}_{\mathrm{im}} \leftarrow B^{(t)}_{\mathrm{re}}\mathbf{x}_{\mathrm{im}} + B^{(t)}_{\mathrm{im}}\mathbf{x}_{\mathrm{re}}$
    \STATE \quad \textit{// Apply scalar scales (broadcast)}
    \STATE \quad $\mathbf{y}^{(t)} \leftarrow s^{(t)}_{\mathrm{re}} \mathbf{v}_{\mathrm{re}} + i \, s^{(t)}_{\mathrm{im}} \mathbf{v}_{\mathrm{im}}$
    \STATE \quad \textbf{atomic add} $\mathbf{y} \leftarrow \mathbf{y} + \mathbf{y}^{(t)}$
\STATE \textbf{end parallel for}
\RETURN $\mathbf{y}$
\end{algorithmic}
\end{algorithm}

\paragraph{Latency and parallelism.}
While recursive residual quantization introduces a summation over $T$ terms ($\mathbf{y} = \sum \widehat{W}^{(t)} \mathbf{x}$), this does not imply a linear increase in latency.
Unlike autoregressive generation, where steps must be sequential, the terms $\widehat{W}^{(t)} \mathbf{x}$ are independent with respect to the input $\mathbf{x}$.
Given sufficient parallel compute units (which GPUs provide in abundance), the critical path latency is:
\begin{equation}
\label{eq:latency}
\mathcal{L}_{\text{total}} = \max_{t} \big( \mathcal{L}(\text{GEMM-free}^{(t)}) \big) + \mathcal{L}(\text{ReduceSum}),
\end{equation}
which is effectively $O(1)$ with respect to $T$. This allows us to use $T=2$ or $T=3$ to drastically improve accuracy without slowing down the inference stream.

\paragraph{LUT-Based optimization for CPU.}
Beyond multiplication-free arithmetic, the discreteness of our 2-bit weights enables highly efficient inference via Look-Up Tables (LUTs), particularly on CPUs~\cite{park2022lutgemm}. Following the principles of optimized kernels like \texttt{BitNet.cpp}~\cite{wang2025bitnetcpp} and T-MAC~\cite{wei2025tmac}, we can pack groups of four 2-bit complex weights into a single 8-bit index. By pre-computing the partial outcomes of these weight combinations with INT8 activations, the inner loop of matrix multiplication reduces to a simple table fetch and accumulation. This approach significantly accelerates execution by minimizing arithmetic instructions while retaining mathematical exactness.

\section{Evaluation}
\label{sec: eval}

In this section, we empirically validate the effectiveness of the \method{} framework. We primarily focus on the LLaMA-2 7B model to demonstrate that our method can successfully restore accuracy at extremely low bit-widths (1-bit and 2-bit effective) by reusing real-valued checkpoints, significantly outperforming real-valued quantization baselines.

\subsection{Experimental Setup}

\paragraph{Models and baselines.}
We conduct experiments on the LLaMA-2 7B model. We compare \method{} against the following baselines:
\begin{itemize}
    \item \textbf{FP16:} The original uncompressed LLaMA-2 7B model.
    \item \textbf{Real-valued QAT Baselines:} To ensure a fair comparison regarding the topological advantage of the complex domain, we implement two strong real-valued QAT baselines trained under the same settings:
    \begin{itemize}
        \item \textit{Real-Binary (1-bit):} Weights are quantized to $\{+1, -1\}$, similar to BitNet~\cite{wang2023bitnet}.
        \item \textit{Real-Ternary (1.58-bit):} Weights are quantized to $\{+1, 0, -1\}$, similar to BitNet b1.58~\cite{ma2024bitnetb1.58}.
    \end{itemize}
    \item \textbf{PTQ Baselines:} We also include results from GPTQ~\cite{frantar2022gptq}, QuIP\#~\cite{tseng2024quipsharp} and AQLM~\cite{egiazarian2024extreme} for reference.
\end{itemize}

\paragraph{Datasets and metrics.}
We evaluate the quantized models on both language modeling capabilities and downstream reasoning tasks using the \texttt{lm-eval-harness} framework~\cite{eval-harness}.
\begin{itemize}
    \item \textbf{Perplexity (PPL):} We report perplexity on the validation set of C4~\cite{raffel2020exploring_c4}.
    \item \textbf{Zero-shot Tasks:} We evaluate zero-shot accuracy on several common sense reasoning benchmarks: 
    ARC-Easy~\cite{yadav2019quick}, ARC-Challenge~\cite{yadav2019quick}, HellaSwag~\cite{zellers2019hellaswag}, PIQA~\cite{bisk2020piqa}, and Winogrande~\cite{sakaguchi2021winogrande}.
\end{itemize}

\paragraph{Implementation details.}
We implement \method{} using PyTorch. The real-valued pretrained weights are first transformed into the widely-linear complex form. We then continue pre-training with the QAT method on a subset of the RedPajama dataset~\cite{weber2024redpajama} for 30 billion tokens. We use the AdamW~\cite{loshchilov2017adamw} optimizer with a Warmup-Stable-Decay (WSD) learning rate schedule~\cite{bi2024deepseek, team2025minicpm4}. The global batch size is set to 1 million tokens. 
Unless otherwise specified, we report results for two configurations of \method{}:
\begin{itemize}
    \item \textbf{\method{}-W1 (1-bit):} Uses the codebook $\mathcal{S}=\{\pm 1, \pm i\}$ for the base complex layer. As shown in Section~\ref{sec:inference-analysis}, this method consumes 1 bit per parameter.
    \item \textbf{\method{}-W2 (2-bit):} Applies one additional step of recursive residual quantization ($T=2$), adding another effective 1 bit per parameter.
\end{itemize}

\subsection{Main Results}

\paragraph{Perplexity evaluation.}
Table~\ref{tab:main_results} presents the perplexity results on C4. \method{} significantly outperforms the real-valued baselines at the same or even higher bit budgets.
Specifically, \method{}-W1 achieves a PPL of 11.03, surpassing the Real-Binary (1-bit) baseline (11.75), and even slightly outperforming the Real-Ternary (1.58 bit) baseline (11.03).
Furthermore, with just one step of recursive residual quantization, \method{}-W2 (2-bit) achieves a remarkable PPL of 7.85, significantly closing the gap to the FP16 baseline (6.63).
It is worth noting that \method{}-W2 outperforms state-of-the-art 2-bit PTQ methods such as AQLM (8.54) and QuIP\# (11.01), and even surpasses 3-bit GPTQ (10.61), demonstrating the efficiency of our residual correction mechanism.

\paragraph{Zero-shot performance.}
Table~\ref{tab:main_results} summarizes the accuracy on five downstream tasks. \method{} demonstrates strong generalization capabilities.
Notably, \method{}-W1 achieves an average accuracy of 48.66\%, outperforming the Real-Binary baseline (46.21\%) at the same effective bit-width, while trailing the Real-Ternary (1.58-bit) model by only a narrow margin (0.04\%).
\method{}-W2 further elevates the performance to 62.00\%, which is highly competitive with the FP16 baseline (64.72\%) and significantly outperforms the leading 2-bit PTQ method AQLM (57.28\%). This confirms that our phase-aware recursive quantization effectively preserves the model's capabilities even under extreme compression.

\begin{table}[htbp]
\centering
\small
\caption{C4 Perplexity and Zero-shot Accuracy comparison on LLaMA-2 7B. Avg. denotes the average accuracy across the 5 tasks. Results for AQLM, QuIP\#, and GPTQ are sourced from~\cite{egiazarian2024extreme}. 
}
\label{tab:main_results}
\resizebox{\textwidth}{!}{
\begin{tabular}{lccccccccc}
\toprule
 &  & \textbf{PPL}$\downarrow$ & \multicolumn{6}{c}{\textbf{Zero-shot Accuracy (\%)}$\uparrow$} \\
\cmidrule(lr){3-3} \cmidrule(lr){4-9}
\textbf{Method} & \textbf{Bits} & \textbf{C4}  & \textbf{ARC-e} & \textbf{ARC-c} & \textbf{HellaSwag} & \textbf{PIQA} & \textbf{Winogrande} & \textbf{Avg.} \\
\midrule
LLaMA-2 (FP16) & 16 & 6.63 & 75.59 & 43.17 & 57.06 & 77.91 & 69.85 & 64.72 \\
\midrule
GPTQ~\cite{frantar2022gptq} & 3 & 10.61 & 58.46 & 31.06 & 45.21 & 71.49 & 59.19 & 53.08 \\
QuIP\#~\cite{tseng2024quipsharp} & 2 & 11.01 & 55.56 & 28.84 & 42.94 & 71.38 & 62.43 & 52.23 \\
AQLM~\cite{egiazarian2024extreme} & 2 & 8.54 & 63.68 & 32.76 & 49.55 & 74.76 & 65.67 & 57.28 \\
\midrule
Real-Binary (QAT) & 1 & 11.75 & 53.32 & 22.70 & 35.57 & 66.81 & 52.64 & 46.21 \\
\textbf{\method{}-W1 (Ours)} & 1 & \textbf{11.03} & \textbf{56.56} & \textbf{24.82} & \textbf{38.19} & \textbf{70.08} & \textbf{53.67} & \textbf{48.66} \\
\midrule
Real-Ternary (QAT) & 1.58 & 11.06 & 55.93 & 24.15 & 38.43 & 69.80 & 55.17 & 48.70 \\
\textbf{\method{}-W2 (Ours)} & 2 & \textbf{7.85} & \textbf{72.73} & \textbf{39.76} & \textbf{53.33} & \textbf{76.17} & \textbf{68.03} & \textbf{62.00} \\
\bottomrule
\end{tabular}
}
\end{table}

\subsection{Ablation Studies}

\paragraph{Effectiveness of recursive residual quantization.}
We analyze the impact of the recursive depth $T$ on model performance. Table~\ref{tab:ablation_recursive} compares \method{}-W1, \method{}-W2, and a further refined version \method{}-W3 ($T=3$), which corresponds to an effective bit-width of 3 bits per parameter.
As observed, increasing $T$ from 1 to 2 results in a substantial reduction in C4 perplexity (20.76\% improvement) and a significant boost in zero-shot accuracy of 19.03\%.
However, upgrading from W2 to W3 yields diminishing returns, with only marginal improvements in PPL (4.11\%) and average accuracy (0.88\%).
Given that W3 increases the memory footprint by 50\% compared to W2, it suggests that $T=2$ is a good tradeoff between model accuracy and storage efficiency for extreme quantization.

\begin{table}[htbp]
\centering
\small
\caption{Ablation study on the number of recursive stages $T$ for \method{} under the learning rate of 5e-4. \method{}-W3 represents using 3 recursive stages with effective 3 bits.}
\label{tab:ablation_recursive}
\resizebox{\textwidth}{!}{
\begin{tabular}{lccccccccccc}
\toprule
 &  & \multicolumn{2}{c}{\textbf{PPL}$\downarrow$} & \multicolumn{7}{c}{\textbf{Zero-shot Accuracy (\%)}$\uparrow$} \\
\cmidrule(lr){3-4} \cmidrule(lr){5-11}
\textbf{Method} & \textbf{Bits} & \textbf{C4} & \textbf{Decrease} & \textbf{ARC-e} & \textbf{ARC-c} & \textbf{HellaSwag} & \textbf{PIQA} & \textbf{Winogrande} & \textbf{Avg.} & \textbf{Increase}\\
\midrule
\method{}-W1 & 1 & 11.03 & - & 56.56 & 24.82 & 38.19 & 70.08 & 53.67 & 48.66 & -\\
\method{}-W2 & 2 & 8.74 & 20.76\% & 69.70 & 33.02 & 48.95 & 74.70 & 63.22 & 57.92 & 19.03\%\\
\method{}-W3 & 3 & 8.38 & 4.11\% & 69.36 & 32.17 & 50.74 & 76.01 & 63.85 & 58.43 & 0.88\%\\
\bottomrule
\end{tabular}
}
\end{table}

\paragraph{Sensitivity to learning rate.}
Training extremely low-bit networks is notoriously unstable and sensitive to hyperparameters. We evaluate the robustness of \method{}-W2 by training with varying learning rate schedules of different decay times.
We use the WSD scheduler. WSD consists of three phases: a linear warmup from a small initial value to a peak, a stable phase that holds the learning rate at the peak for a period to stabilize training, and an optional cosine decay phase to reduce the learning rate later. To study sensitivity to the LR magnitude, we evaluate three WSD variants (denoted as LR1, LR2, and LR3). \textbf{LR1}: linear warmup to a peak of $3\times 10^{-5}$ in 50 steps, then stable at the peak for the remainder of training with no decay applied. \textbf{LR2}: same as LR1 but decay to $5\times 10^{-6}$ from step 9000 within 2000 steps. \textbf{LR3}: same as LR2 but ecay to $1\times 10^{-6}$ from step 19000 within 1000 steps.

Figure~\ref{fig:ablation_lr} illustrates the training dynamics.
Figure~\ref{fig:ablation_lr}(a) depicts the Warmup-Stable-Decay (WSD) learning rate schedules employed during training.
Figure~\ref{fig:ablation_lr}(b) plots the training loss curves corresponding to different learning rate schedules.
As observed, our method is robust across different learning rate schedules. Furthermore, the effectiveness of the WSD scheduler is clearly visible: as training progresses into the decay phase, the reduction in learning rate enables the model to converge to a lower final loss, confirming that a well-scheduled decay is crucial for optimizing quantized models.
Table~\ref{tab:ablation_lr} shows the final PPL and task performance with different learning rate schedules.

\begin{figure}[htbp]
  \centering
  \begin{subfigure}[b]{0.35\textwidth}
    \includegraphics[width=\linewidth]{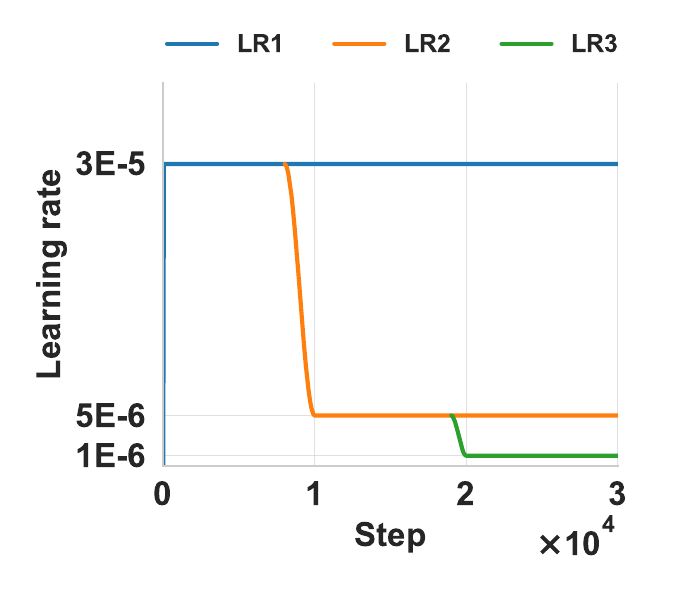} 
    \caption{Learning Rate Schedule}
    \label{fig:lr_schedule}
  \end{subfigure}
  \hspace{0.7cm}
  \begin{subfigure}[b]{0.35\textwidth}
    \includegraphics[width=\linewidth]{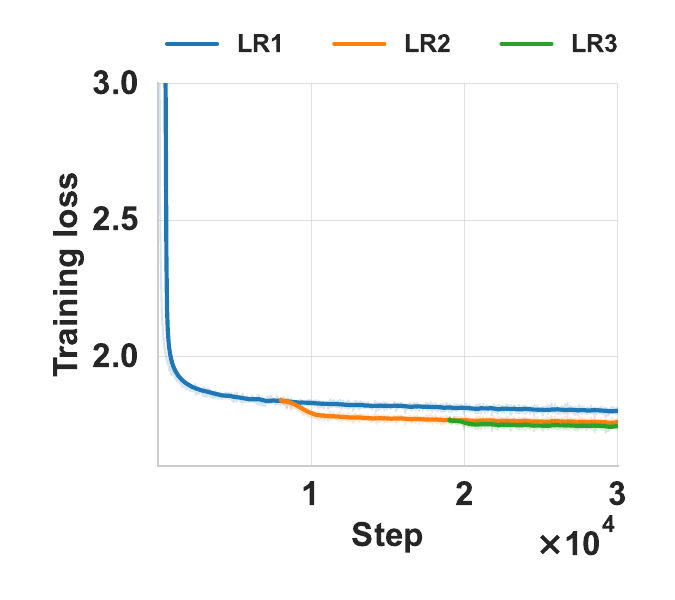}
    \caption{Training Loss Curves}
    \label{fig:loss_curves}
  \end{subfigure}
  \caption{Sensitivity analysis of \method{} to learning rate. (a) The WSD learning rate schedules. (b) Training loss curves under different learning rate schedules.}
  \label{fig:ablation_lr}
\end{figure}

\begin{table}[htbp]
\centering
\small
\caption{Ablation study on the learning rate schedule with \method{}-W2.}
\label{tab:ablation_lr}
\resizebox{\textwidth}{!}{
\begin{tabular}{lcccccccc}
\toprule
 & & \multicolumn{1}{c}{\textbf{PPL}$\downarrow$} & \multicolumn{6}{c}{\textbf{Zero-shot Accuracy (\%)}$\uparrow$} \\
\cmidrule(lr){3-3} \cmidrule(lr){4-9}
\textbf{Method} & \textbf{Bits} & \textbf{C4} & \textbf{ARC-e} & \textbf{ARC-c} & \textbf{HellaSwag} & \textbf{PIQA} & \textbf{Winogrande} & \textbf{Avg.}\\
\midrule
\method{}-W2-LR1 & 2 & 8.42 & 69.49 & 36.60 & 51.25 & 75.03 & 65.04 & 59.48 \\
\method{}-W2-LR2 & 2 & 7.98 & 72.39 & 38.23 & 53.13 & 76.01 & 66.85 & 61.32\\
\method{}-W3-LR3 & 2 & 7.85 & 72.73 & 39.76 & 53.33 & 76.17 & 68.03 & 62.00 \\
\bottomrule
\end{tabular}
}
\end{table}

\section{Conclusion}
\label{sec: conclusion}

We propose \method{}, a universal framework that bridges the representational capacity of complex-valued LLMs with the practical utility of pre-trained real-valued LLMs. By deriving an exact, widely-linear representation, we enable the seamless conversion of real-valued layers into the complex domain, allowing for the reuse of existing checkpoints without prohibitive retraining. 
Leveraging a phase-aware codebook $\{\pm 1, \pm i\}$ and a recursive residual quantization mechanism, \method{} maximizes the information density of extremely low-bit budgets. 
Experiments on LLaMA-2 7B demonstrate that our approach outperforms state-of-the-art real-valued binary and ternary methods, restoring performance to near-FP16 levels at an effective 2-bit precision.

\paragraph{Future work.}
Several promising avenues remain for future exploration. First, we aim to develop specialized CUDA kernels and CPU optimization techniques to fully exploit the multiplication-free nature of our quantization scheme, accelerating both training and inference on commodity hardware. Second, we plan to scale \method{} to larger foundation models (e.g., LLaMA-3 70B) and multimodal architectures to verify its universality. Third, further theoretical investigation into the complex-valued loss landscape could provide deeper insights into the robustness observed in low-bit regimes. 
Finally, and perhaps most critically, our current implementation was limited to 30 billion tokens due to computational constraints. We hypothesize that the complex-valued representation possesses a superior capacity that has yet to be fully exploited; with extended training on larger datasets, we believe \method{} holds the potential to not merely match, but surpass the accuracy of the original full-precision baselines.

\printbibliography
\end{document}

%% file: share-preamble.tex
\providecommand{\method}{\textsc{Fairy2i}}
\providecommand{\RePart}{\operatorname{Re}}
\providecommand{\ImPart}{\operatorname{Im}}

\usepackage{mathtools}
\usepackage{newtxtext,newtxmath}

\usepackage{amsthm}
\usepackage{bm}
\usetikzlibrary{arrows.meta, positioning, calc, shapes.geometric, shadows, backgrounds, fit}

\definecolor{paperBlue}{RGB}{218, 232, 252}
\definecolor{paperRed}{RGB}{248, 206, 204}
\definecolor{paperGray}{RGB}{245, 245, 245} 
\definecolor{paperYellow}{RGB}{255, 242, 204}
\definecolor{borderGray}{RGB}{102, 102, 102}
\definecolor{arrowColor}{RGB}{120, 120, 120} 
\definecolor{codeGreen}{RGB}{213, 232, 212}

%% file: Figure/Fairy2i-transform.tex
\begin{tikzpicture}[>=Latex, font=\large]

\tikzset{
    param/.style={
        draw=borderGray, thick, 
        minimum width=2.2cm, minimum height=1.4cm, 
        align=center, rounded corners=2pt,
        font=\Large,
        drop shadow={opacity=0.15, shadow xshift=2pt, shadow yshift=-2pt} %
    },
    io/.style={
        align=center,
        font=\Large,
        inner sep=5pt,
        rounded corners=2pt,
        fill=white, %
        text=black!80
    },
    sum/.style={
        draw=borderGray, circle, inner sep=0pt, minimum size=0.8cm, thick,
        fill=white,
        append after command={
            node at (\tikzlastnode.center) {+}
        }
    },
    mapbox/.style={
        font = \Large,
        draw=orange!60!gray, thick, %
        align=center, 
        fill=paperYellow!40, %
        drop shadow={opacity=0.2},
        rounded corners=4pt, 
        inner sep=8pt
    },
    relationship/.style={
        dashed, ultra thick, %
        color=arrowColor!80, 
        ->, 
        rounded corners=5pt,
        line cap=round
    },
    dataflow/.style={
        ->, thick, color=black!90
    }
}

\def\colLeft{-4.0}  
\def\colMid{0}      
\def\colRight{4.0}  
\def\colOutput{8.5} 
\def\yTop{4.5}      %
\def\yBottom{-5}  %

\node[mapbox, minimum width=5.5cm, minimum height=3.5cm] 
    (datamap) at (\colLeft, \colMid) 
    {
        \textbf{Input Transformation}\\
        $\tilde{\mathbf{x}} = \begin{bmatrix} \RePart\mathbf{x} \\ \ImPart\mathbf{x} \end{bmatrix}$\\
        $\uparrow\downarrow$\\
        $\mathbf{x} = \RePart\mathbf{x} + i \ImPart\mathbf{x}$
    };

\node[mapbox, minimum width=3.5cm, minimum height=3.8cm] 
    (parammap) at (\colRight, \colMid) 
    {
        \textbf{Parameter Transformation}\\[5pt]
        $\RePart\mathbf{U} = \frac{1}{2}(\mathbf{R}_{11} + \mathbf{R}_{22})$\\[5pt]
        $\ImPart\mathbf{U} = \frac{1}{2}(\mathbf{R}_{21} - \mathbf{R}_{12})$\\[5pt]
        $\RePart\mathbf{W} = \frac{1}{2}(\mathbf{R}_{11} - \mathbf{R}_{22})$\\[5pt]
        $\ImPart\mathbf{W} = \frac{1}{2}(\mathbf{R}_{12} + \mathbf{R}_{21})$\\[5pt]
        {\normalsize (Deriving $\mathbf{U}, \mathbf{W}$ from $\mathbf{R}$)}
    };

\node[font=\bfseries\Large, color=black!80] at (\colMid+4, \yTop + 1.5) {Real Linear Layer};

\node[io] (xin_real) at (\colLeft+2, \yTop) {$\tilde{\mathbf{x}} \in \mathbb{R}^{2m}$};

\node[param, fill=paperGray] (R_mat) at (\colRight, \yTop) {$\mathbf{R}=\begin{bmatrix} \mathbf{R}_{11} & \mathbf{R}_{12} \\ \mathbf{R}_{21} & \mathbf{R}_{22} \end{bmatrix}$};

\node[io] (yout_real) at (\colOutput+2, \yTop) {$\tilde{\mathbf{y}} = \mathbf{R}\tilde{\mathbf{x}} \in \mathbb{R}^{2n}$}; 

\draw[dataflow] (xin_real) -- (R_mat);
\draw[dataflow] (R_mat) -- (yout_real);

\node[font=\bfseries\Large, color=black!80] at (\colMid+4, \yBottom - 2.2) {Widely Linear Complex Layer};

\node[io] (xin_cplx)  at (\colLeft+2, \yBottom + 1.0) {$\mathbf{x}$};
\node[io] (xin_cplx_conj) at (\colLeft+2, \yBottom - 1.0) {$\overline{\mathbf{x}}$};

\node[param, fill=paperBlue] (U_mat) at (\colRight, \yBottom + 1.0) {$\mathbf{U}=\RePart\mathbf{U}+i\,\ImPart\mathbf{U}$};
\node[param, fill=paperRed]  (W_mat) at (\colRight, \yBottom - 1.0) {$\mathbf{W}=\RePart\mathbf{W}+i\,\ImPart\mathbf{W}$};

\node[sum] (sum_node) at (\colRight + 3.5, \yBottom) {}; 

\node[io] (yout_cplx) at (\colOutput+2, \yBottom) {$\mathbf{y} = \mathbf{U}\mathbf{x} + \mathbf{W}\overline{\mathbf{x}} \in \mathbb{C}^n$}; 

\draw[dataflow] (xin_cplx) -- (U_mat);
\draw[dataflow] (xin_cplx_conj) -- (W_mat);
\draw[dataflow] (U_mat.east) -| (sum_node.north);
\draw[dataflow] (W_mat.east) -| (sum_node.south);
\draw[dataflow] (sum_node) -- (yout_cplx); 

\draw[dashed, color=gray!80, ultra thick] (xin_cplx) -- node[left, font=\Large, text=gray] {conj} (xin_cplx_conj);

\draw[relationship] (xin_real.west) -| (datamap.north);
\draw[relationship] (datamap.south) |- (xin_cplx.west);

\draw[relationship] (R_mat.south) -- (parammap.north);

\draw[relationship] (parammap.south) -- (U_mat.north);
\draw[relationship] (parammap.south west) to[out=250, in=110] (W_mat.north west);

\node[mapbox, minimum width=5.0cm, minimum height=3.5cm] 
    (outmap) at (\colOutput+4.5, \colMid) 
    {
        \textbf{Output Transformation}\\
        $\tilde{\mathbf{y}} = \begin{bmatrix} \RePart\mathbf{y} \\ \ImPart\mathbf{y} \end{bmatrix}$ \\
        $\uparrow\downarrow$\\
        $\mathbf{y} = \RePart\mathbf{y} + i \ImPart\mathbf{y}$
    };

\draw[relationship] (yout_real.east) -| (outmap.north); 

\draw[relationship] (outmap.south) |- (yout_cplx.east);

\end{tikzpicture}

%% file: Figure/recursive-darker.tex
\definecolor{deepBlue}{RGB}{0, 85, 164}       %
\definecolor{fillBlue}{RGB}{235, 242, 250}    %

\definecolor{deepRed}{RGB}{178, 34, 34}       %
\definecolor{fillRed}{RGB}{255, 240, 240}     %

\definecolor{deepYellow}{RGB}{210, 140, 0}    %
\definecolor{fillYellow}{RGB}{255, 250, 235}  %

\definecolor{deepGreen}{RGB}{34, 139, 34}     %
\definecolor{fillGreen}{RGB}{235, 250, 235}   %

\definecolor{darkGray}{RGB}{60, 60, 60}       %
\definecolor{lightGray}{RGB}{245, 245, 245}   %

\begin{tikzpicture}[
    >={Latex[length=2.5mm, width=1.5mm]}, %
    font=\sffamily %
]

\tikzset{
    process/.style={
        draw=deepYellow, thick, 
        align=center, rounded corners=3pt,
        fill=fillYellow,
        minimum width=2.5cm, minimum height=1.2cm,
        font=\bfseries\sffamily,
        drop shadow={opacity=0.2, shadow xshift=1pt, shadow yshift=-1pt}
    },
    store/.style={
        draw=deepGreen, thick,
        align=center, rounded corners=2pt,
        fill=fillGreen,
        minimum width=1.5cm, minimum height=1cm,
        font=\large\bfseries
    },
    io/.style={
        align=center, font=\Large\bfseries, 
        text=black
    },
    op/.style={
        draw=black, circle, thick, 
        inner sep=0pt, minimum size=0.6cm,
        fill=white, font=\Large\bfseries
    },
    plotbox/.style={
        draw=darkGray!60, thin,
        fill=white,
        rounded corners=2pt,
        minimum size=2.2cm
    },
    mainflow/.style={->, very thick, color=darkGray, rounded corners=5pt},
    resflow/.style={->, very thick, color=deepBlue!80, rounded corners=5pt},
    splitline/.style={-, thick, color=darkGray}
}

\def\xStart{0}
\def\xStageOne{4}
\def\xStageTwo{12}
\def\xEnd{14}

\def\yRes{2}    
\def\yQuant{0}  
\def\yAcc{-3.5}

\node[io] (input) at (\xStart, \yRes) {$\mathbf{R}^{(0)}=\mathbf{W}$};
\node[font=\sffamily, color=deepGreen] at (\xStart, \yRes+0.75) {Full Precision};

\node[op] (sub1) at (\xStageOne + 2.5, \yRes) {$-$};

\draw[resflow] (input) -- (sub1);

\node[io] at (\xStageOne + 2.5, \yRes + 0.75) {$\mathbf{R}^{(1)} = \mathbf{R}^{(0)} - \widehat{\mathbf{W}}^{(0)}$};

\node[process] (quant0) at (\xStageOne, \yQuant) {Phase\\Quant};

\draw[mainflow] (\xStageOne, \yRes) -- (quant0);
\draw[fill=deepBlue] (\xStageOne, \yRes) circle (2pt); 

\node[plotbox, left=0.8cm of quant0] (plot0) {};
\begin{scope}[shift={(plot0.center)}]
    \draw[->, gray!60] (-0.8,0) -- (0.8,0);
    \draw[->, gray!60] (0,-0.8) -- (0,0.8);
    \draw[gray!40] (0,0) circle (0.6cm);
    \foreach \ang in {90}
        \fill[deepRed] (\ang:0.6) circle (2pt);
    
    \draw[->, thick, deepBlue] (0,0) -- (60:0.7);
    \draw[->, very thick, deepRed] (0,0) -- (90:0.6);
\end{scope}
\draw[->, thin, darkGray] (plot0.east) -- (quant0.west);

\node[store] (what0) at (\xStageOne, \yQuant-2) {$\widehat{\mathbf{W}}^{(0)}$};
\draw[mainflow] (quant0) -- node[right, font=\large] {$s^{(0)}_{\mathrm{re}}, s^{(0)}_{\mathrm{im}}$} (what0);

\draw[resflow] (what0.east) -| (sub1.south);

\node[process] (quant1) at (\xStageTwo, \yQuant) {Phase\\Quant};

\draw[mainflow] (sub1.east) -- (\xStageTwo, \yRes) -- (quant1.north);

\node[plotbox, left=0.8cm of quant1] (plot1) {};
\begin{scope}[shift={(plot1.center)}]
    \draw[->, gray!60] (-0.8,0) -- (0.8,0);
    \draw[->, gray!60] (0,-0.8) -- (0,0.8);
    \draw[gray!40] (0,0) circle (0.6cm);
    \fill[deepRed] (0:0.6) circle (2pt);
    \draw[->, thick, deepBlue] (0,0) -- (20:0.5);
    \draw[->, very thick, deepRed] (0,0) -- (0:0.6);
\end{scope}
\draw[->, thin, darkGray] (plot1.east) -- (quant1.west);

\node[store] (what1) at (\xStageTwo, \yQuant - 2) {$\widehat{\mathbf{W}}^{(1)}$};
\draw[mainflow] (quant1) -- node[right, font=\large] {$s^{(1)}_{\mathrm{re}}, s^{(1)}_{\mathrm{im}}$} (what1);

\node[op] (sum) at (\xEnd-2, \yAcc) {$+$};

\draw[mainflow] (what0) |- (sum.west);
\draw[mainflow] (what1) |- (sum.north);

\node[io] (output) at (\xEnd, \yAcc) {$\mathbf{W}_{q}$};
\draw[mainflow] (sum) -- (output);

\node[font=\sffamily, color=deepGreen, align=center] at (\xEnd, \yAcc - 0.75) {Deployed\\Weight};

\begin{scope}[on background layer]
    \node[fit=(quant0)(what0)(plot0), fill=fillRed, draw=deepRed!30, thick, rounded corners, inner sep=10pt] (box0) {};
    \node[anchor=south west, font=\bfseries\sffamily, color=deepRed, xshift=5pt] at (box0.south west) {Stage 0 ($t=0$)};
    
    \node[fit=(quant1)(what1)(plot1), fill=fillRed, draw=deepRed!30, thick, rounded corners, inner sep=10pt] (box1) {};
    \node[anchor=south west, font=\bfseries\sffamily, color=deepRed, xshift=5pt] at (box1.south west) {Stage 1 ($t=1$)};
\end{scope}

\node[anchor=north, align=left, font=\sffamily, draw=darkGray!50, thick, rounded corners, fill=white] 
    at (6, \yAcc-0.5)
    {
        \textcolor{deepBlue}{$\bm{\rightarrow}$} Input/Residual Weight\\
        \textcolor{deepRed}{$\bm{\rightarrow}$} Quantized code $\{\pm 1, \pm i\}$
    };

\end{tikzpicture}